\documentclass[10pt,journal,compsoc]{IEEEtran}
\ifCLASSOPTIONcompsoc

  \usepackage[nocompress]{cite}
\else
  \usepackage{cite}
\fi

\ifCLASSINFOpdf

\else

\fi

\usepackage[utf8]{inputenc}
\usepackage[english]{babel}
\usepackage{hyperref}
\usepackage{soul}
\usepackage[dvipsnames]{xcolor}
\usepackage{amsmath}
\usepackage{cleveref}
\usepackage{ragged2e}
\usepackage{subcaption} 

\usepackage{pifont}
\newcommand{\cmark}{\ding{51}}%
\newcommand{\xmark}{-}

\usepackage{amsfonts}

\usepackage{amsthm}

\usepackage{xspace}

\makeatletter
\@mparswitchfalse%
\makeatother

\newcommand{\fref}[1]{Fig.~\ref{#1}}
\newcommand{\tref}[1]{Table~\ref{#1}}

\newcommand{\provtype}[1]{\ensuremath{\mathsf{\prov{#1}}}}
\newcommand{\codeformat}[1]{\texttt{\small #1}}
\newcommand{\nodename}[1]{\texttt{#1}}

\newcommand{\pg}{Pok\'emon Go}
\newcommand{\pokemon}{Pok\'emon}
\newcommand{\pokemons}{Pok\'emons}
\newcommand{\pokestop}{Pok\'eStop}
\newcommand{\pokestops}{Pok\'eStops}

\newcommand{\textlabel}[1]{\textsf{#1}}

\usepackage{listings}
\usepackage{caption}
\newcounter{nalg} % defines algorithm counter for chapter-level
\lstnewenvironment{algorithm}[1][] %defines the algorithm listing environment
{   
    \refstepcounter{nalg} %increments algorithm number
    \captionsetup{labelsep=colon}
    \lstset{ %this is the stype
        mathescape=true,
        frame=tB,
        numbers=left, 
        numberstyle=\small,
        basicstyle=\normalsize, 
        keywordstyle=\color{black}\bfseries,%\em,
        keywords={,input, output, return, datatype, function, in, if, else, foreach, while, for, begin, end, repeat, until, to, } %add the keywords you want, or load a language as Rubens explains in his comment above.
        numbers=left,
        xleftmargin=.04\textwidth,
        #1 % this is to add specific settings to an usage of this environment (for instnce, the caption and referable label)
    }
}{}

\usepackage[inline]{enumitem} % To play with items naming
\usepackage{bm}

\newcounter{questionchapter}
\newcounter{question}[questionchapter]
\renewcommand\thequestion{\textbf{E\arabic{question}}}
\newenvironment{questions}[1][]{\refstepcounter{questionchapter}\refstepcounter{question}\par\medskip
\begin{list}{\thequestion}{\usecounter{question}}}{\end{list}\medskip}

\newtheorem{example}{Example} 
\newtheorem{theorem}{Theorem}
 
\newtheorem{proposition}[theorem]{Proposition}

\newtheorem{definition}[theorem]{Definition}

\newcommand{\ent}{\textit{ent}}
\newcommand{\ag}{\textit{ag}}
\newcommand{\act}{\textit{act}}

\newcommand{\wgb}{gen}
\newcommand{\used}{use}
\newcommand{\waw}{assoc}
\newcommand{\wdf}{der}
\newcommand{\spec}{spe}
\newcommand{\wib}{wib}
\newcommand{\wat}{wat}
\newcommand{\abo}{abo}
\newcommand{\alt}{alt}

\newcommand{\Lab}{lab}
\newcommand{\lab}{\textit{l}}
\newcommand{\LAB}{\mathcal{L}}

\newcommand{\real}{\mathbb{R}}

\newcommand{\dotproduct}[1]{\left\langle #1 \right\rangle}

\newcommand{\sz}[1]{\left\vert{#1}\right\vert}

\newcommand{\labwalk}{LAB}

\newcommand{\kernelfunction}{k}
\newcommand{\kernelset}{\mathcal{X}}
\newcommand{\embeding}{\psi}
\newcommand{\hilbert}{\mathcal{H}}

\newcommand{\ftype}{\phi}
\newcommand{\aftype}{F}
\newcommand{\ftypevector}{\text{VEC}}
\newcommand{\fkernel}{k}
\newcommand{\depth}{h}

\newcommand{\ang}[1]{\langle#1\rangle}

\newcommand{\qname}[1]{\codeformat{#1}}
\newcommand{\person}{\nodename{person13}}
\newcommand{\patzero}{\nodename{patient7$_0$}}

\newcommand{\pattwo}{\nodename{patient$7_2$}}
\newcommand{\patthree}{\nodename{patient$7_3$}}
\newcommand{\wardone}{\nodename{ward27}}

\newcommand{\actone}{\nodename{admitting3}}
\newcommand{\acttwo}{\nodename{treating5}}
\newcommand{\discharged}{\qname{mimic:DischargedPatient}}

\newcommand{\edgeset}{\tau}

\newcommand{\GG}{\mathcal{G}}
\newcommand{\vv}{\mathcal{V}}
\newcommand{\ee}{\mathcal{E}}
\newcommand{\nodelabels}{S}
\newcommand{\familynodelabels}{\mathcal{\nodelabels}}
\renewcommand{\provtype}[1]{\Lab(#1)}  % Label of a node

\newcommand{\OO}{\mathcal{O}}

\newcommand{\wl}{WL\xspace}
\newcommand{\textdef}[1]{\textit{#1}}

\usepackage{tikz} 
\usepackage{tikzpagenodes}
\usepackage{color,soul}
\usepackage{graphs}
\usepackage{booktabs}
\usepackage{nicefrac}
\usepackage[normalem]{ulem}
\useunder{\uline}{\ul}{}

\begin{document}
\title{Provenance Graph Kernel}

\author{David Kohan Marzagão, Trung Dong Huynh, Ayah Helal, Sean Baccas, Luc Moreau
% ,~\IEEEmembership{Member,~IEEE,}
%         John~Doe,~\IEEEmembership{Fellow,~OSA,}
%         and~Jane~Doe,~\IEEEmembership{Life~Fellow,~IEEE}% <-this % stops a space
\IEEEcompsocitemizethanks{\IEEEcompsocthanksitem D. Kohan Marzagão is with the Department
of Engineering, University of Oxford, UK. He is also affiliated to King's College London.\protect\\
% E-mail: see http://www.michaelshell.org/contact.html
\IEEEcompsocthanksitem T.D. Huynh, and L. Moreau are with Department of Informatics, King's College London, UK. \protect\\
\IEEEcompsocthanksitem A. Helal is with Department of Computer Science, Exeter University, UK. She is also affiliated to King's College London.\protect\\
\IEEEcompsocthanksitem S. Baccas is with Polysurance.}% <-this % stops an unwanted space
\thanks{This work is supported by a Department of Navy award (Award No. N62909-18-1-2079) issued by the Office of Naval Research. The United States Government has a royalty-free license throughout the world in all copyrightable material contained herein.}
}


\markboth{Journal of \LaTeX\ Class Files,~Vol.~14, No.~8, August~2015}%
{Shell \MakeLowercase{\textit{et al.}}: Bare Demo of IEEEtran.cls for Computer Society Journals}



\IEEEtitleabstractindextext{%
\begin{abstract}
Provenance is a record that describes how entities, activities, and agents have influenced a piece of data;
it is commonly represented as graphs with relevant labels on both their nodes and edges.
With the growing adoption of provenance in a wide range of application domains, users are increasingly confronted with an abundance of graph data, which may prove challenging to process.
Graph kernels, on the other hand, have been successfully used to efficiently analyse graphs.
In this paper, we introduce a novel graph kernel called \emph{provenance kernel}, which is inspired by and tailored for provenance data.
It decomposes a provenance graph into tree-patterns rooted at a given node and considers the labels of edges and nodes up to a certain distance from the root.
We employ provenance kernels to classify provenance graphs from three application domains.
Our evaluation shows that they perform well in terms of classification accuracy and yield competitive results when compared against existing graph kernel methods and the provenance network analytics method while more efficient in computing time.
Moreover, the provenance types used by provenance kernels also help improve the explainability of predictive models built on them.
\end{abstract}

}

\maketitle

\IEEEdisplaynontitleabstractindextext

% \fi
\IEEEpeerreviewmaketitle

\IEEEraisesectionheading{\section{Introduction}\label{sec:introduction}}

\IEEEPARstart{P}{rovenance}
% Provenance 
is a form of knowledge graph providing an account of the actions a system performs, describing the data involved and the processes carried out over those data.
More specifically, the World Wide Web Consortium (W3C) defines provenance as a \textit{record that describes the people, institutions, entities, and activities involved in producing, influencing, or delivering a piece of data or a thing in the world} \cite{w3c-prov-dm}. 
Provenance is increasingly being captured in a variety of application domains, from scientific workflows~\cite{Alper2013,chirigati2013reprozip} supporting their reproducibility, to climate science~\cite{ma2014capturing}, and human-agent teams in disaster response~\cite{Ramchurn2016}.
However, simultaneously, users are being confronted with an increasing volume of provenance data, especially from automated systems, making it a challenge to extract meaning and significance.
In this paper, we are interested in the comparison of provenance graphs leading to their classification according to their similarities, with means to procedurally extract explanations for a given classification decision. 

Similarities between graphs have been often studied in the context of graph classification and detection of malicious activity, with a plethora of different methods being proposed to that end (for recent surveys on graph kernels, see~\cite{kriege2020survey}, \cite{nikolentzos2019survey}, and~\cite{graph-kernels-challenges}).
Such methods explore graph properties using concepts such as shortest paths between nodes~\cite{Borgwardt2005}, sub-trees~\cite{shervashidze2011weisfeiler, feragen2013scalable}, or random walks~\cite{gartner2003graph}, to mention just a few.
In many cases, the graphs being analysed may have continuous or categorical labels for their nodes or edges, but little attention has been given to developing a graph kernel that considers \emph{both} edge and node categorical labels.
A well-known family of graphs that can benefit from such a kernel method is the one of provenance graphs.

To that end, and inspired by the expressiveness of edge and node labels in provenance graphs conditioned to their distance to a given root node, we introduce a graph kernel method that not only captures provenance graph patterns efficiently, but is explainable, i.e., can be used to help generate explanations of decisions based on them.
To extract features from such graphs, our method makes use of the notion of \textdef{provenance types}. 
Intuitively, provenance types are a simplification of tree-patterns of a given depth $\depth$ rooted at a given node. It captures the set of edge-labels occurring at each layer of these tree-patterns, giving an unpolished indication of the node's past history. Types also take in account the labels of nodes at the leaves of such tree-patterns. For each graph, a feature vector is created: it counts the occurrence of each provenance type encountered in this graph for tree-patterns up to a given depth. 
We show that the computational complexity of inferring types up to depth $\depth$ of nodes in a graph with $m$ edges is bounded by $\OO(\depth ^2 m)$.

We employ provenance kernels in classifying provenance graphs in six different data sets from three application domains.
We compare the performance of provenance kernels against that of existing graph kernels in those classification tasks, showing that provenance kernels are competitive in terms of accuracy and at the same time as being fast in terms of execution times.
Compared to Provenance Network Analytics (PNA)~\cite{Huynh2018}, a provenance-specific graph classification method, provenance kernels outperform it both in terms of accuracy and computation time.

Furthermore, provenance types employed by provenance kernels (to represent features of provenance graphs) are shown to enable users to gain insights into why a classification was made.
We present an example of how the importance of each provenance type can be inferred from a trained classifier using LIME~\cite{lime}, identifying the most influential types with respect to a particular classification task.
For each identified provenance type, a verbal description can be generated computationally to facilitate the explaining of a classification decision.
%Further, we show that such explanations are in line with our prior knowledge about the data set. 
% provenance kernels can also be employed by white-box classifications models, i.e.\ interpretable predictive models where the model provides an explanation for its classification.
% We discuss the use of white-box models with provenance types and show how they can improve the interpretability of classification tasks.}

In summary, the contribution of this paper is twofold: 
\begin{enumerate}
    \item The definition, implementation, and evaluation of a novel graph kernel method, i.e.\ provenance kernels, which are shown to perform well in classifying provenance graphs when compared to standard graph kernels and the PNA method.
    % \item In conjunction with either interpretable machine learning (ML) models or explainers, provenance types help improve the explainability of classification decisions based on them.
    \item Illustrating how provenance types can help improve the explainability of classification models built with provenance kernels.
\end{enumerate}  

In the remainder of the paper, Section~\ref{sec:framework} introduces provenance kernels, including an efficient algorithm to infer feature vectors from input graphs to be used for their classification.
The related work is then discussed in Section~\ref{sec:related_work} with a comparison of the theoretical computational complexity between provenance kernels and existing graph kernels.
Section~\ref{sec:evaluation} presents the empirical evaluation of provenance kernels in six classification tasks in comparison with generic graph kernels and the PNA method.
Section~\ref{sec:explainability} then discusses the use of provenance types in conjunction with LIME to explain classification decisions.
Finally, Section~\ref{sec:conclusion} concludes the paper and outlines the future work.

\section{Provenance Kernel Framework}\label{sec:framework}
    %!TEX root = main.tex

In this section, we motivate and present a graph kernel method inspired by particular characteristics of provenance graphs, namely, the chronological aspect represented by relations in such structures. We base our definitions on the PROV data model \cite{w3c-prov-dm}, a \textit{de jure} standard for provenance data.
We first lay out the provenance foundations and the graph kernel  concepts that we will use throughout the paper.
We then motivate the idea of provenance types, present the algorithm to infer them, analysing its theoretical computational complexity, to finally provide a formal definition of provenance kernel.

    %!TEX root = main.tex
\subsection{Preliminaries}

The main node and edge labels of the PROV framework, as well as their notation to be used throughout this paper, is presented in \tref{fig:labels}.
The three node labels shown at the top of the table are called \textdef{PROV generic} labels because they are used irrespective of particular provenance applications.
The labels of start and destination nodes connected by edges of each given edge label are also specified in, respectively, the third and fourth columns.
For a complete description of the PROV data model, refer to \cite{w3c-prov-dm}.

We denote $G = (V, E, \nodelabels, L)$ a \textdef{provenance graph} in which $V$ corresponds to the set of nodes of $G$ with $\sz{V} = n$, $E$ its set of its directed edges, with $\sz{E} = m$. The sets of labels of nodes and edges of $G$ are denoted, respectively, by $\nodelabels$ and $L$. An edge $e \in E$ is a triplet $e = (v, u, \lab)$,  where $v \in V$ is its \textdef{starting point}, $u \in V$ is its \textdef{ending point}, and $\lab \in L$, also denoted $\Lab(e)$, is the edge's \textdef{label}.\negmedspace\footnote{In the absence of ambiguity, we will abuse notation and refer to edges as simply pairs $(v,u)$} Note that defining edges as triplets instead of pairs allow provenance graphs to have more than one edge between the same pair of vertices.
Each node $v \in V$ can have more than one label, and thus $\provtype{v} \in \mathbb{P}(\nodelabels)$, where $\mathbb{P}(\nodelabels)$ denotes the power-set of $\nodelabels$, i.e., the set of subsets of $\nodelabels$.
Note that provenance graphs are \textdef{finite}, \textdef{directed}, and \textdef{multi-graphs} (as there might exist more than one edge between the same pair of nodes).

Considering only PROV generic labels, $\nodelabels = \{\ag, \act, \ent\}$. In case application-specific labels are used, however, set $\nodelabels$ is enlarged to also include such specific labels.  
Regarding edge-labels, typically $L = \{\wdf, \spec, \alt, \wib, \wgb, \dots \}$, where the edge $(v,u,\waw)$, for example, indicates that activity $v$ was associated with agent $u$. 
We will often work with more than one graph, and thus we define $\GG = (\vv, \ee, \familynodelabels, \LAB)$ as a (finite) family of graphs, in which $\vv$, $\ee$, $\familynodelabels$, and $\LAB$, are the union of the sets of, respectively, nodes, edges, node labels, and edge labels of graphs in $\GG$.
    \begin{table}
        \centering
        \begin{tabular}{@{}cccc@{}}
        \toprule
             \textbf{Label} & \textbf{Notation} & \textbf{Source Type} & \textbf{Destination Type}  \\ 
             \midrule
             Agent & \ag   & - & - \\ 
             Activity & \act & - & -  \\ 
             Entity & \ent  & - & -  \\ 
             \midrule
             wasDerivedFrom & \wdf & \ent & \ent \\ 
             specializationOf & \spec & \ent & \ent  \\ 
             alternateOf  & \alt  & \ent & \ent \\ 
             wasInvalidatedBy & \wib & \ent & \act \\ 
             wasGeneratedBy & \wgb & \ent & \act \\ 
             used & \used & \act & \ent \\ 
             wasAttributedTo & \wat & \ent & \ag  \\ 
             wasAssociatedWith & \waw & \act & \ag \\ 
             actedOnBehalfOf & \abo & \ag & \ag \\ 
             wasStartedBy & wsb & \act & \ent  \\ 
             wasEndeddBy & web & \act & \ent  \\ 
             wasInformedBy & wifb & \act & \act \\ 
            \bottomrule
        \end{tabular}
    \caption{PROV generic labels for nodes (first three rows) and edges, and their
    notation used 
    throughout this paper. The third and fourth columns show, respectively, the label of source and destination nodes for each edge label. 
    }
    \label{fig:labels}
\end{table}

In this work, we will study provenance kernels in both scenarios: when application-specific labels are provided and when they are not.
In terms of the generality of graph structures considered in this paper, however, observe that existence of cycles may not be discarded entirely: edges such as usage and generations may create cycles, as well as future invalidation of entities.
This is to say that, although edges, for well defined provenance, in general `point to the past', cycles cannot be totally excluded and, for that reason, our definitions will make no restrictive assumptions on graph properties, such as acyclicity.
Indeed, the definitions that follow apply to a general graph with labelled edges and nodes. 

The \textdef{forward-neighbourhood} of a given node $v \in \vv$ is the set of nodes it ``points to'', i.e., $v^+ = \{u \mid (v,u, \lab) \in \ee\}$. Analogously, the \textdef{backward-neighbourhood} of $v$ is denoted by $v^- = \{u \mid (u, v, \lab) \in \ee\}$.
We say a node $u$ is \textdef{distant from} $v$ by $x$ if there is a \textdef{walk} from $v$ to $u$ of length $x$, where $v$ is the walks's \textdef{starting} node, and $u$ its \textdef{ending} node. That is, there is a sequence of $x$ (not necessarily distinct) consecutive edges starting at $v$ and ending at $u$. More formally, a sequence $(e_1, \dots, e_s)$ is of consecutive edges iff for $1 \leq i < s$, the pair $e_i = (v_i, u_i, \lab_i)$ and $e_{i+1} = (v_{i+1}, u_{i+1}, \lab_{i+1})$ is such that $u_i = v_{i+1}$. 
The $x+1$ nodes in a walk of length $x$ need not to be distinct. A \textdef{path}, on the other hand, is defined as a walk where nodes do not repeat. 

A function $\kernelfunction: \kernelset \times \kernelset \to \real$ is called a valid \textdef{kernel} on set $\kernelset$ if there is a real Hilbert space $\hilbert$ and a mapping $\embeding$ such that $\kernelfunction(x,y) = \dotproduct{\embeding(x), \embeding(y)}$. In order to show such a Hilbert space exists (and therefore $\kernelfunction$ is called its \textit{reproducing kernel}), it is enough to prove that $\kernelfunction$ is symmetric and positive semi-definite (p.s.d.) \cite[Theorem 3]{berlinet2011reproducing}, i.e., for every subset $\{x_1, \dots, x_t\} \subset \kernelset$, we have that the $t \times t$ 
matrix $M$ defined by $M(i,j) = \kernelfunction(x_i, x_j)$ is p.s.d. For $M$ to be p.s.d, we simply need $\sum_{i,j} c_ic_j M(i,j) \geq 0$ for all $c_1, \dots, c_t \in \real$.

    %!TEX root = main.tex

\subsection{Motivation}

Consider a node $v$ in a provenance graph. The set of edges starting at this node may be seen as related to its recent history, i.e., such edges point to other nodes that may, in case of entities, be the activity that generated it, or, in case of activities, be the agent responsible for its execution. Going further, edges that are, in turn, connected to the neighbours of $v$, represent $v$'s earlier history, and so on. The idea of capturing the label of such edges, as well of nodes, taking in account their distance to the root, lies in the core of what we define as provenance types. Provenance types are the building blocks for provenance kernels.

For example, refer to \fref{fig:mimic_example}. It depicts an example of a short patient hospitalisation, from admission to discharging. Here, $V = \{\person, \patzero, \wardone, \dots \}$, while $\nodelabels = \{\ent, \ag, \act\}$. If application-specific types are used, $S$ is enlarged to include labels such as \qname{mimic:Patient}, \qname{mimic:Ward}, etc.
We adopt the de facto colour and layout convention for provenance graphs that shows entities as a yellow-filled ellipses, activities as blue-filled rectangles, and agents as orange-filled trapeziums.
Time flows downwards in this convention, in which the entities $\patzero, \dots, \patthree$ represent the different `states' of the same person originally represented by node $\person$, culminating at $\patthree$ which also contains the provenance-specific label $\discharged$, indicating that the hospitalisation ended with the discharging of the patient.
The activities in this scenario are those that either admit the patient to a ward ($\actone$) or represent a treatment ($\acttwo$). Each is associated to the respective hospital ward in which the activity took place.

As a motivation for provenance types, consider nodes $\actone$ and $\acttwo$ in \fref{fig:mimic_example}. We can say that they share some similarity as both represent activities in this provenance graph, even though one is an admission and the other a treatment. Further, we can say that they share even more similarities as they are related to entities (via the \textit{\used} relation) and to some agent (via the \textit{\waw} edge label). Going one step further, however, these nodes do not present the same `history':  $\acttwo$ used an entity which was, in turn, generated by some other activity, whereas $\actone$ did not.
We formalise this idea of capturing a simplification of a node's history as the \textdef{provenance types} of a node.
First, we define label-walks, which will be later used in our definition of types.

\begin{figure}[t]
    \centering
    \def\svgwidth{0.49\textwidth}
    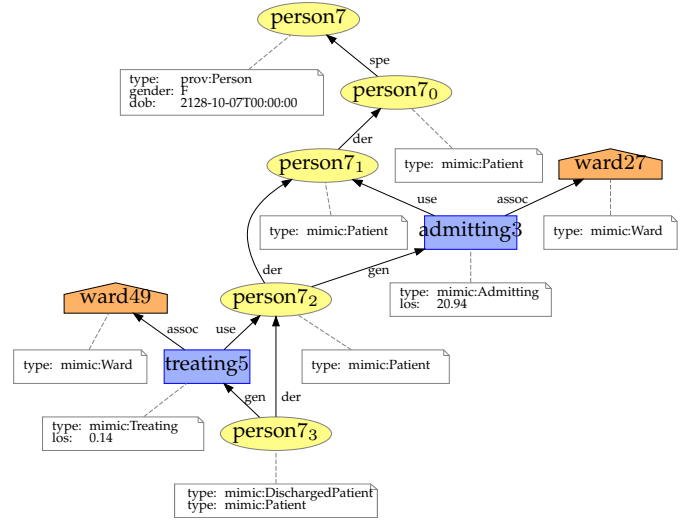
    \caption{A provenance graph that records the journey of a patient through a hospital admission. Yellow ellipses denote entities, blue rectangles denote activities, and orange trapeziums denote agents. The same patient is recorded multiple times via different entities to represent their different states over time.}
    \label{fig:mimic_example}
\end{figure}

\begin{definition}[Label-walk]
    Let $G = (V, E)$ be a directed multigraph. We define a label-walk as 
    \begin{equation}
        \labwalk(w) = (\Lab(e_1), \Lab(e_2), \dots, \Lab(e_{\depth}), \provtype{u})
    \end{equation}
    where $w =(e_1, \dots, e_{\depth})$ is a sequence of $\depth$ consecutive edges from $v$ to $u$.
    $\labwalk(w)$ is then a sequence of $\depth$ edge labels followed by the label of vertex $u \in V$. Thus, we say that $\labwalk(w)$ is a label-walk of length $\depth$. 
    We further define $\mathcal{W}^\depth(v)$ to be the set of all label-walks of length $\depth$ starting at a given node $v$. In the degenerated case of $\depth = 0$, the start and end nodes coincide and thus $\mathcal{W}^0(v)$ is defined by simply $(\provtype{v})$.
\end{definition}

    \noindent In simple terms, a label-walk is the sequence of the labels of edges along a walk followed by the label of its ending node. The intuition behind capturing the label of ending nodes of walks is twofold. First, we are able to define a base case, in which we consider a walk of null size, i.e., capturing only the node's label. Secondly, we are able to make use of application-specific node labels (such as \qname{mimic:Patient}, in \fref{fig:mimic_example}), which may provide crucial information for the analysis of provenance data.

\begin{example}\label{exm:label-walk}
    In this example we consider only PROV generic types. Consider the sequence of two consecutive edges given by
         $w =$ \Big((\patthree, \acttwo, \wgb), 
         (\acttwo, \pattwo, \used)\Big).  
    % \end{equation}
   We have $\labwalk(w) = (\wgb, \used, \ent)$. Moreover, 
       $\mathcal{W}^2(\patthree) =$ \{(\wgb, \used, \ent), (\wgb, \waw, \ag), \\ 
       (\wdf, \wdf, \ent), (\wdf, \wgb,\act)\}.
\end{example}
    We can now define provenance types based on the set of all walks of a given length from each node in $G$. 
    
\begin{definition}[Provenance $\depth$-types] \label{def:prov_types}
    Let $G$ be a graph and $v \in V$. Let $\mathcal{W}^\depth(v)
    $ be the set of all label-walks of length $\depth$ starting at $v$. We now capture all the labels that are equally distant from $v$: we define, for $1 \leq i \leq \depth$,
    \begin{equation}
    \begin{split}
              \edgeset^{\depth}_i = \{ \Lab(e_{\depth - i}) \mid (\Lab(e_1), \Lab(e_2), \dots, \Lab(e_{\depth}),  \\ 
              \provtype{u}) \in \mathcal{W}^\depth(v)\} 
    \end{split}
    \end{equation} 
    as the set of edge labels that are in the 
    $(\depth-i)$-th layer when counting from $v$.\negmedspace\footnote{This apparent reversed choice of indexing will simplify the algorithm to infer such provenance types. The intuition is that, with this notation, $\edgeset^{\depth}_0$ always refers to node labels for any  $\depth$.} For $i =0$, we define 
    \begin{equation}
    \begin{split}
        \edgeset^{\depth}_0 = \{ \Lab(u) \mid (\Lab(e_1), \dots, \Lab(e_{\depth}),  \provtype{u}) \in \mathcal{W}^\depth(v)\}
        \end{split}
    \end{equation}
    We define the provenance $\depth$-type
    of $v$ as the sequence of sets
    \begin{equation}
       \ftype^{\depth}(v) = (\edgeset^{\depth}_\depth, \dots, \edgeset^{\depth}_0)
    \end{equation}
    In the case $\mathcal{W}^\depth(v) = \emptyset$, i.e. there is no walk of length $\depth$ from $v$, we define $\ftype^{\depth}(v) = \emptyset$. When clear from the context, we shall denote $\edgeset^{\depth}_i(v)$ as simply $\edgeset_i(v)$, or even $\edgeset_i$.
\end{definition}

\begin{example}\label{exm:f-types}
    Consider $\patthree$ discussed in Example \ref{exm:label-walk}. Its $2$-type combines the label-walks in $\mathcal{W}^2(\patthree)$
    to generate: 
    \begin{multline}
        \ftype^{2}(\patthree) = \\ = (\underbrace{\{\wgb, \wdf\}}_{\edgeset^{2}_2},
        \underbrace{\{\used, \waw, \wdf, \wgb\}}_{\edgeset^{2}_1}, 
        \underbrace{\{\ag, \act, \ent\}}_{\edgeset^{2}_0}) 
    \end{multline}
    Further, consider nodes $\actone$ and $\acttwo$ once again in \fref{fig:mimic_example}. Their $0$-types, $1$-types, and $2$-types are:
    \begin{align*}
        &\ftype^{0}(\actone) = \ftype^{0}(\acttwo) = (\{act\}) \\
        &    \ftype^{1}(\actone) = \ftype^{1}(\acttwo) = (\{\waw, use\}, \\ & \qquad \qquad \qquad \qquad \qquad \qquad \qquad \qquad \qquad \qquad \{act, ag\})
        \\
        &\ftype^{2}(\actone) = (\{use\}, \{der\},\{ent\}) \\
        &\ftype^{2}(\acttwo) = (\{use\}, \{der, gen\},\{act, ent\})
    \end{align*}
    These were both examples using only PROV generic types. By using application-specific types for the same two nodes $\actone$ and $\acttwo$, on the other hand, we have
    $(\{act, \textit{mimic:Admitting}\}) = \ftype^{0}(\actone) \neq \\ \ftype^{0}(\acttwo) = (\{act, \textit{mimic:Treating}\}). $
\end{example}

\noindent Note that two different sets of label-walks may give rise to the same provenance type, although the converse is not true, i.e., two different types cannot come from the same set of label-walks. This implies that the function that maps sets of label-walks to types is not, in general, injective. We claim that this is beneficial as it unifies different sets of label-walks that have very similar meaning in provenance. Note also that multiple occurrences of the same walks starting at $v$ carry out no difference in $v$'s provenance types as opposed to just one copy of each different walk. 
    %!TEX root = main.tex

\subsection{Algorithm}

We now present an algorithm that infers all $\ftype^{i}$ for $0\leq i \leq \depth$ for nodes in a family of graphs $\mathcal{G}$ in $O(\depth^2M)$, where $M$ is the total number edges among graphs in $\mathcal{G}$. For a single graph, the algorithm runs in $O(\depth^2m)$, where $m$ is the number of edges in this particular graph.

        \subsubsection{The Algorithm}
        \fref{fig:atypes_algorithm} provides an algorithm to infer provenance $\depth$-types. First, we initialise all types $\ftype^i(v)$ with the empty set for all depths up to $\depth$ and all nodes. Moreover, we initialise as empty sets the building blocks of our $\depth$-types that record the labels of edges seen in at a given distance from each node (lines $1$-$3$). The intuition behind this explicit initialisation is that if we do not update a given $\ftype^i(v)$, the empty set will indicate that there are no label-walks of size $i$ starting at $v$. This will be used later as a condition in line $8$. 
        
        The loop starting at line $4$ infers the base of our algorithm: the $0$-type of all nodes, i.e., $\ftype^0(v)$, which is simply the set of labels of $v$ for each $v \in \mathcal{V}$.
        Each iteration of the loop starting in line $7$ will infer the $i$-type for all nodes. We first loop through all edges in $\mathcal{E}$ that can `lead us somewhere'. In other words, we are considering only edges $e = (v,u)$ that belong to walks of size $i$ starting at $v$. This is true if and only if $\ftype^{i-1}(u) \neq \emptyset$. We then make sure that label of $e$ is added to the set $\edgeset^i_i(v)$ (line $9$).
        Finally, the loop starting at line $10$ adds to the set of $\edgeset^i_j(v)$ the labels from set $\edgeset^{i-1}_j(u)$. We can finally in line $12$ construct the entire $i$-type for all nodes. 
        
        To see that the algorithm correctly infers provenance types, note that line $8$ guarantees that all label-walks of length $i$ starting at $v$ will be identified. Further, that lines $10$ and $11$ make sure that all labels in each of these label-walks will be fully inspected and added to $v$'s type accordingly. 
         
    \begin{figure}
            \centering        
            \begin{algorithm}[caption={$\textsc{Inferring types up to }\depth \,\,(\vv, \ee, \depth)$}, label={alg_dijkstra}]
initialise for all $v \in \mathcal{V}$ and for all $i \leq \depth$ 
    $\ftype^i(v) \leftarrow \emptyset$ 
    for all $0\leq j \leq i$,  $\edgeset_j^i(v)\leftarrow \emptyset$
for $v \in \mathcal{V}$
    $\edgeset_0^0(v) \leftarrow \Lab(v)$
    $\ftype^0(v) \leftarrow (\edgeset_0^0(v))$
for $1 \leq i \leq \depth$
    for each $e = (v,u) \in \mathcal{E}$ s.t. $\ftype^{i-1}(u) \neq \emptyset$
        add $\Lab(e)$ to $\edgeset^i_i(v)$
        for  $0 \leq j \leq i-1$
            $\edgeset_j^i(v)\leftarrow \edgeset_j^i(v) \cup \edgeset_j^{i-1}(u)$
    $\ftype^{i}(v) \leftarrow (\edgeset^i_{\depth}(v), \dots, \edgeset^i_0(v))$
return $\ftype^i(v)$ for all $v \in \vv$ and $0 \leq i \leq k$
          \end{algorithm}
            \caption{An algorithm that receives nodes and edges of a family of graphs, a parameter $\depth$, and outputs all provenance types $\ftype^0(v), \dots, \ftype^{\depth}(v)$ for all nodes $v$.}
            \label{fig:atypes_algorithm}
        \end{figure}

\subsubsection{Complexity Analysis}
        
            We are now showing that we need $O(\depth^2M)$ operations to infer the $\ftype^{\depth}(v)$ for each node $v$ in a family of graphs $G_1, \dots, G_s := \mathcal{G}$. Here, $N$ stands for the sum of the number of nodes in all provenance graphs and $M$ for the sum of the number of edges.
            We borrow part of the argument from \cite{shervashidze2011weisfeiler}.
            % Similarly to the evaluation of Weisfeiler-Lehman graph kernels in \cite{shervashidze2011weisfeiler}, w
            % We can infer the types of nodes on each graph in parallel. 
            Lines $1$-$3$ can be done in $O(\depth^2 N)$, since we are initialising $\frac{1}{2}(\depth + 1)(\depth + 2)$ sets for each node in $\mathcal{V}$. Lines $4$-$6$ take $O(N)$. Let us now investigate the \textit{for} loop initiated in line $7$. We are entering this loop $\depth$ times, and in each of them we are investigating each edge at most once (loop stating in line $8$), and finally, for each edge, we are performing at most $\depth$ pairwise operations on sets of constant size (bounded by $\max\{\sz{T}, \sz{L}\}$). Line $12$ takes $O(N)$. Thus loop starting at line $7$ can be done in $O(\depth^2M)$, assuming $N = O(M)$. Which gives us the overall running time bounded by $O(\depth^2M)$ when we assume $N = O(M)$.

    %!TEX root = main.tex

\subsection{Provenance Kernel}

In this section, we define the mapping of graphs into a high dimensional space by simply counting the number of occurrences of each provenance $\depth$-types up to depth $\depth$. We formally define a feature vector in the following definition.

\begin{definition}[Feature Vector]\label{def:feature_vector}
    Let $\GG$ be a family of graphs and define $\ftype^{\depth}(\vv) = \{\aftype_1, \dots, \aftype_s\}$ as the (enumerated) set of all provenance types of depth \textbf{up to} $\depth$ encountered in $\vv$. The feature vector of a graph $G \in \GG$ is given by:
    \begin{equation}
        \ftypevector^{\,h}(G) = (x_1, x_2, \dots, x_s)
    \end{equation}
    where $x_i = \sz{\{v \mid \ftype^{\depth}(v) = \aftype_i, \text{ and } v \in G\}}$
\end{definition}
Note that Definition~\ref{def:feature_vector} explicitly considers the possibility parallel inference of feature vectors among all graphs in a family $\GG$, as previously suggested by \cite{shervashidze2011weisfeiler}.
% in the context of of the number of edges, similarly to the evaluation of Weisfeiler-Lehman graph kernels. 
In fact, this property comes from the fact that provenance kernels are an explicit graph kernel, i.e., the feature vectors, and not only the dot products between each pair, are known (for other examples of explicit graph kernels, see \cite[p. 14]{nikolentzos2019survey}). We apply this definition to our graph in \fref{fig:mimic_example} as an example.

\begin{example}
    Consider the provenance graph $G$ depicted in \fref{fig:mimic_example}. In order to infer $\ftypevector^{\,1}(G)$, we need $G$'s $0$-types and $1$-types. There are three of the former and four of the latter, and we denote them by:
    \begin{align*}
        &F_1 = (\{\act\}), \qquad F_2  = (\{\ag\}), \qquad F_3 = (\{\ent\}), \\
        &F_4 = (\{\waw, \used\}, \{\act, \ag\}),\qquad F_5 = (\{\spec\}, \{\ent\}), \\
        &F_6 = (\{\wdf\}, \{\ent\}), \qquad
        F_7 = (\{\wgb, \wdf\}, \{\act, \ent\})
    \end{align*}
    And thus 
    \begin{equation}
        \ftypevector^{\,1}(G) = (5, 2, 2, 2, 1, 1, 2)
    \end{equation}
\end{example}
We now use the definition of a feature vector to define provenance kernels. 
\begin{definition}[Provenance Kernel]
    Given two graphs $G, G' \in \GG$ and $\ftype^i(\vv)$ for all $0 \leq s \leq {\depth}$, we define the kernel between $G$ and $G'$ as 
    \begin{equation}
        \fkernel^\depth(G, G') = \sum_{s=0}^\depth \ang{\ftypevector^s(G), \ftypevector^s(G')} 
    \end{equation}
    where $\ang{x,y}$ denotes the dot product between $x$ and $y$.
\end{definition}

\begin{proposition}
    Provenance kernels are p.s.d.
\end{proposition}
\begin{proof}
    Let $c_1, \dots, c_t \in \real$ and $G_1, \dots, G_t \in \GG$. Consider for a given depth $s$ and pair of indices $i,j$, the dot product $\ang{\ftypevector^s(G_i), \ftypevector^s(G_j)}$. Then, 
    % \begin{multline}
        $\sum_{i=1}^t\sum_{j=1}^t c_ic_j \ang{\ftypevector^s(G_i), \ftypevector^s(G_j)} = \\
        \ang{\sum_{i=1}^t c_i \ftypevector^s(G_i), \sum_{j=1}^t c_j\ftypevector^i(G_j)} \geq 0$.
    % \end{multline}
    The inequality follows from the fact that both sums add to exactly the same value and from $\dotproduct{x,x} \geq 0$ for all $x$. Since sum of non-negative numbers is non-negative, $\fkernel^\depth$ is p.s.d.
\end{proof}

\section{Related Work}\label{sec:related_work}
    %!TEX root = main.tex

Similar definitions of provenance types have been proposed as a tool for provenance graph summarisation \cite{Moreau2015, ipaw_library}. In both these approaches, the idea of the history of provenance nodes being related to a sequence of transformations described by edge labels is used. To the best of our knowledge, however, this is the first work to explore the use of provenance types in the context of machine learning methods.
When comparing the efficiency of the other summarisation methods with our own, we find that the definition by \cite{Moreau2015} requires an exponential time $\OO(nd^\depth)$, where $d$ is the maximum degree of nodes in an input graphs, and $n$ its number of nodes.
On the other hand, \cite{ipaw_library} propose a faster algorithm that takes $\OO(\depth m)$, where $m$ is the number of edges of an input graph. This is faster than our algorithm by a factor of $\depth$, which is typically small and does not depend on the size of the graph.
This faster algorithm, however, shares similarities with Weisfeiler-Lehman graph kernels to a point in which patterns with very close meaning in provenance are classified differently. More specifically, the algorithm does not inspect sub-trees beyond their first level in order to discard repetitions. This is discussed in detail later in this section and exemplified in \fref{fig:patterns}.

ML techniques 
on graphs 
have been proposed in the domain on provenance (e.g. .
Provenance Network Analytics (PNA) \cite{Huynh2018}, for example, creates, for each graph, a feature vector that encodes a sequence of different graph topological properties.
Some of these are provenance agnostic, such as the number of nodes, or the number of edges. Others, in contrast, record the longest shortest paths between, e.g., two provenance entities, or between an agent and an activity.
In Section~\ref{sec:evaluation}, we compare the performance of provenance kernels and PNA in the same classification tasks. For other works in provenance and ML see \cite{souza:lirmm-03324881} and \cite{miao2017towards}. 

Provenance kernels compare and classify graphs, as opposed to comparing and classifying nodes on graphs. The latter is often known as kernels \emph{on graphs}  (as opposed to \emph{graph kernels}), or \emph{graph embedding} techniques. A notable example of graph embedding technique is Struc2Vec \cite{struc2vec} is a learning technique based on the structural identity of nodes on graphs that embed nodes into a Euclidean space according to the structure of their neighbourhood. Similarly to provenance kernels, struc2vec has a hierarchical approach when looking at the structure of the neighbourhood of nodes. The main difference to our work is that Struc2Vec does not consider edge nor node labels, but instead the number of occurrences of neighbouring elements (and their degrees). Another commonality is that both works define a distance between nodes based on their neighbours and a suitable metric.
In the context of knowledge graphs, graph embedding for link prediction  was used in \cite{rosso2020www}.
The former work used implicit computations to compare a pair of graphs, whether the latter introduced explicit methods allowing for faster computation (for a survey on similar models, see \cite{ristoski2016semantic}). 
ML methods for Resource Description Framework (RDF) graphs have been studied by \cite{ristoski2016rdf2vec} (RDF2Vec), \cite{loschgraphkernelsrdf} and \cite{de2015substructure}.
Other known approaches that for example aim to find missing labels or links in graphs include NodeSketch \cite{yang2019nodesketch}, DeepWalk \cite{10.1145/2623330.2623732}, 
and Node2Vec \cite{10.1145/2939672.2939754}.

Using the classification from \cite{graph-kernels-challenges}, we can place provenance kernels in the \textit{Information Propagation} graph kernel family. More specifically, in the group of \textit{Kernels based on iterative label refinement}. That is because we iteratively update node labels up to $\depth$ times based on the other node labels. Other graph kernel methods in the same family include Neighbourhood Hash kernels (NH) \cite{Neighbourhood-Hash-Kernel}, Weisfeiler-Lehman kernels (WL) \cite{shervashidze2011weisfeiler}, and Neighborhood Subgraph Pairwise Distance kernel (NSPD) \cite{neighborhood-subgraph-pairwise}. Out of those, only the latter originally have taken into account the edge labels in addition to node labels in its classification. A possible extension for labelled edges, however, is mentioned in WL original work \cite{shervashidze2011weisfeiler} and made explicit by~\cite{de2013fast}.

On the other hand, the \wl graph kernel algorithm presented by~\cite{shervashidze2011weisfeiler} does not consider edge labels.
A close variant to this WL extension was also proposed by~\cite{ipaw_library}, with the differences that (1) repetitions of branches from the same given parent in its walk-tree are discarded and; (2) although all edge-labels were considered, only nodes at the leaves of trees had their labels taken in account. Another close variant of WL, that uses truncated trees as features, was proposed recently in \cite{8862853}.
Discarding repetitions, however, is key in reducing the sizes of feature vectors by putting together patterns that have a similar meaning in provenance. In this paper, we extend this process further by agglutinating even more similar patterns into the same provenance type. 
\fref{fig:patterns} shows two examples of such types.
In essence, the extra sequence of two derivations in Pattern $1$ does not add qualitative meaning to the set of transformations that lead to the creation of the root node. 
Thus, provenance kernels consider these two patterns as having the same type. One can see provenance types for provenance kernels as a \emph{flattened} version of the ones defined by~\cite{ipaw_library} since the idea of branches is hidden by the sole enumeration of labels at a given depth. 

\begin{figure}
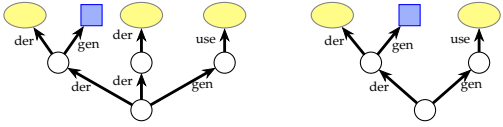

    \centering
  \begin{subfigure}[b]{0.23\textwidth}
    \begin{center}
        \wldiffnewone
    \end{center}
    % \caption{Pattern 1}
    \label{fig:1}
  \end{subfigure}
  \begin{subfigure}[b]{0.17\textwidth}
    \begin{center}
        \wldiffnewtwo
    \end{center}
    % \caption{Pattern 2}
    \label{fig:2}
  \end{subfigure}
  \caption{Two examples of different tree-patterns that have a similar meaning in provenance. 
  Yellow circles denote entities, blue squares denote activities, and plain circles denote nodes of which types are not captured by the algorithm. Provenance kernels classify both these tree-patterns as the same $2$-type, i.e., $(\{\wdf, \wgb\}, \{\wdf, \wgb, \used\},\{\act, \ent\}) $.}
  \label{fig:patterns}
\end{figure}

In terms of theoretical computational complexity, provenance kernels are situated in the efficient spectrum. Note that provenance kernel's computational complexity for one graph with $m$ edges, $\OO(\depth ^2 m)$, is bounded by $\OO(\depth ^2 n d)$, where $d$ is the maximum (out) degree and $n$ is the number of nodes of this graph. Also, $m = nd$ if and only if the input graph is regular.
Graphlet Sampling kernel (GS)~\cite{graphlet} counts the number of small subgraphs present in each input graph.
These small subgraphs typically have size $k \in \{3,4,5\}$. The original time required to run this kernel, $\OO(n^k)$, is prohibitively expensive.
Later, \cite{shervashidze2009efficient} improved this bound to $\mathcal{O}(nd^{k-1})$, where $d$ is the maximum degree of nodes in an input graph.
Neighbourhood Hash kernel (NH)~\cite{Neighbourhood-Hash-Kernel} compares graphs by counting the number of common node labels, which are updated with the employment of logical operations that take in account the label of neighbouring nodes.
These operations do not use the original categorical node label, but a binary array so XOR operations can be employed. The complexity of this kernel is bounded by $\mathcal{O}(bhn\bar{D})$, where $\bar{D}$ is the average degree of the vertices, $h$ is how many times the update is executed, and $b$ is the number of bit labels.
We need $2^b - 1>> |\Sigma|$, size of the label set. 
Finally, Neighborhood Subgraph Pairwise Distance kernel (NSPD)~\cite{neighborhood-subgraph-pairwise} compares subgraphs $S$ in the neighbourhood of nodes up to a distance $r$ for all pair of nodes in a given graph.
The complexity, $\mathcal{O}(n |S| |E(S)| \log |E(S)|)$, is such that $S$ is the subgraph induced by the neighbourhood of a vertex and radius $r$. $E(S)$ is the number of edges of $S$.

We can see that some of the theoretical gaps on running times between provenance kernels and the others discussed above come down to low-impact variables that may even be treated as constants, such as the number of bit labels or the size of graphlets. Others, such as maximum degrees and number of edges of neighbourhood subgraphs, are more dependent on input graphs and tend to have more impact on running times, which indicate that provenance kernels are computed faster.
These differences are reflected in the empirical running times evaluated in Section~\ref{sec:evaluation}, in which we show that provenance kernels outperform the three methods described above (GS, NH, and NSPD) in terms of efficiency.
Moreover, none of these is designed to take as input the edge categorical labels of graphs.
Subgraph Matching kernel, on the other hand, does consider both edge and node labels as it counts the number of common subgraphs of bounded size $k$ between two graphs.
It requires, however, an impractical computational running time of $\OO(kn^{k-1})$, where $n$ this time stands for the sum of sizes of the two graphs being compared.
In fact, this kernel timed out in our experiments and its accuracy could not be measured.

\begin{table*}
    \centering
    \caption{The graph kernels evaluated in Section~\ref{sec:evaluation} and their properties: whether node and edge categorical labels are considered and their theoretical computational complexity. Symbol $\approx$ indicates that the need of labels depend on the choice of base kernel.}
    \begin{tabular}{@{}cccc@{}}
    \toprule
         Kernel Method  & Node label & Edge label & Complexity \\ 
        \midrule
         Provenance Kernels (PK) & \cmark & \cmark & $\mathcal{O}(\depth^2 m)$ \\
         Shortest Path (SP) & \cmark & \xmark & $\mathcal{O}(n^4)$ \\ 
         Vertex Histogram (VH) & \cmark & \xmark & $\mathcal{O}(n)$\\
         Edge Histogram (EH) & \xmark & \cmark & $\mathcal{O}(m)$\\
         Graphlet Sampling (GS) & \xmark & \xmark & $\mathcal{O}(nd^{k-1})$ \\
         Hadamard Code (HC) & \cmark & $\approx$ & $\mathcal{O}(hm)$\\
         Weisfeiler-Lehman (WL)  & \cmark & $\approx$ & $\mathcal{O}(\depth m)$ \\
         Neighbourhood Hash (NH) & \cmark & \xmark & $\mathcal{O} (hm)$ \\ 
         Neigh. Subgraph P. Dist. (NSPD)  & \cmark & \cmark & $\mathcal{O}(n |S| |E(S)| \log |E(S)|)$ \\
         \bottomrule
    \end{tabular}
    \label{tab:other_kernels}
\end{table*}

\tref{tab:other_kernels} presents a summary of all graph kernels to be compared to provenance kernels in Section~\ref{sec:evaluation}.
In this table, we note whether node or edge categorical labels were taken into account in the implementation we used, as well as their theoretical complexity.
We now briefly discuss the remaining kernel methods.
The Shortest Path kernel (SP)~\cite{Borgwardt2005}, for each graph, constructs a new graph that captures the original graph's shortest paths and then uses a base kernel to compare two shortest path graph's.
The complexity of SP is $\mathcal{O}(n^4)$. We use the algorithm of Vertex Histogram (VH) and Edge Histogram (EH) kernels as presented by \cite{sugiyama2015halting}. VH creates, for each graph, a feature vector that captures the number of nodes with each given node label $\lab$, whereas EH does the analogous for edge labels. Their computational complexities are $\mathcal{O}(n)$ for VH and  $\mathcal{O}(m)$ for EH. Note that provenance kernels of depth $0$ coincide with VH.
Hadamard Code kernel (HC)~\cite{kataoka2016hadamard} is similar to NH, even by showing the same computational complexity.
It explores the neighbourhood of nodes iteratively for different levels (or depths). Its name come from the use of Hadarmard code matrices.

Finally, when compared against the PNA method, provenance kernels outperform in terms of theoretical computational complexity. One of the features considered in the PNA method is the longest shortest path between two nodes of a given label (two entities, or one entity and one activity, and so on). This gives us a lower bound for PNA's computational complexity of $\Omega(nm)$.

\section{Empirical Evaluation}\label{sec:evaluation}
    %!TEX root = main.tex

As a tool for analysing provenance graphs, provenance kernels are suitable for
ML techniques over provenance data. To demonstrate the approach, we
employ provenance kernels in classification tasks on six provenance
data sets (see below). In our evaluation, we compare the
accuracy of provenance kernels against generic graph kernels and the PNA method
(discussed in Section~\ref{sec:related_work}) in the same classification tasks.
We describe the evaluation methodology in Section~\ref{sub:methodology} and
report the evaluation's results in Section~\ref{sub:results}.
    %!TEX root = main.tex

\subsection{Data sets}\label{sub:datasets}

We employed six provenance data sets in our evaluation; they were produced by
three different applications: MIMIC~\cite{MIMIC-III},
CollabMap~\cite{Ramchurn2013}, and a \pg\ simulator. These applications cover a
spectrum of human and computational processes. The first, MIMIC, records solely
human activity; the second, CollabMap, is created with computational workflows
driven by human inputs, whereas \pg\ is a fully synthetic system.
% We now provide an overview of these applications.

\subsubsection{MIMIC Application}\label{ssub:mimic}

MIMIC-III~\cite{MIMIC-III} is an openly available data set comprising
de-identified health data associated with over 53,000 intensive care unit
admissions at a hospital in the United States. It contains details collected
from hospital stays of over 30,000 patients, including their vital signs and
medical measurements, their diagnostics, the procedures carried out on them and
by whom, etc. In this application, we use the data from MIMIC-III to
reconstruct a patient's journey through the hospital in a provenance graph (see
\fref{fig:mimic_example} for an example). Each admission starts with the patient
being admitted to a unit at the hospital and followed by transfers from one
unit to another. The patient before and after a stay in the care of a ward/unit
are modelled as two separate entities in the provenance graph, the latter
derives from the former as a result of the corresponding `treatment' activity
associated with the respective ward/unit. In addition, each \emph{procedure}
(e.g.\ inserting peripheral lines, imaging, ventilation) that was carried out on
the patient is similarly modelled as an activity with two entities to represent
the patient before and after the procedure. Both types of activities can
happen in parallel. We also annotate procedure activities with their procedure
types, e.g.~\nodename{process:225469}, specifying what the procedures are. There
are 116 different types of procedures recorded in the data set; out of those, 8
types are in the ``Communication'' category, e.g.~meeting the family, notifying
the family. Since those do not clinically affect a patient, we ignore them in
our analyses, leaving 108 procedure types recorded in the provenance graphs we
produced.

Approximately 10\% of admissions resulted in in-hospital
mortality. For each hospital admission, we associate its provenance graph with a
\codeformat{dead} label, which has either a value of 0 or 1, with 1 representing
in-hospital mortality. We then aim to predict the mortality result from a
patient's journey during admission by applying the provenance kernels over
its provenance graphs. This set of graphs is now called \codeformat{MIMIC}.

\subsubsection{CollabMap Application}\label{ssub:collabmap}

CollabMap~\cite{Ramchurn2013} is a crowdsourcing platform for constructing
evacuation maps for urban areas. In these maps, evacuation routes connect exits
of buildings to the nearby road network. Such routes need to avoid obstacles that
are not explicit in existing maps (e.g.\ walls or fences).
CollabMap crowdsources the drawing of such evacuation routes from the
public by providing them aerial imagery and ground-level panoramic views of an
interested area. It allows non-experts to perform tasks without them needing
expertise other than drawing lines on an image. The task of identifying all
evacuation routes for a building was broken into micro-tasks performed by
different contributors: building identification (outline a \emph{building}),
building verification (vote for the building's correctness), route identification
(draw an evacuation \emph{route}), route verification (vote for the correctness
of routes), and completion verification (vote on the completeness of the current
\emph{route set}). This setup allows individual contributors to verify each
other's contributions (i.e.\ buildings, routes, and route sets).

In order to support auditing the quality of its data, the provenance of crowd
activities in CollabMap was fully recorded: the data entities that were shown to
users in each micro-task, the new entities generated afterwards, and their
dependencies. The provenance graphs from CollabMap are, therefore, recorded the
actual activities as they occurred and are not reproduced after the fact
(see~\cite{Huynh2018} for more details).
A notable point of this data set is that the provenance graph associated with a
data entity does not describe its history but the following, later entities and 
activities that depended on it. Hence, given an entity, such a graph was called
a (provenance) \textdef{dependency graph} of that entity~\cite{Huynh2018}, which
is analogous to a graph detailing the citations of an academic paper, their
later citations, and so on. More formally, the dependency graph of a node \(v\)
in a provenance graph $G$ is the sub-graph of $G$ induced by all nodes from
which it is possible to reach $v$, i.e., the sub-graph induced by the set of
nodes $S = \{u \mid \text{there exists a walk from } u \text{ to } v\} $.

In 2012, CollabMap was deployed to help map the area around the Fawley Oil
refinery in the United Kingdom. It generated descriptions for 5,175 buildings,
4,997 routes, and 4,710 route sets. In this application, we aim to predict the
quality of CollabMap data entities from their provenance dependency graphs,
i.e.~whether a building, route, or route set is sufficiently trustworthy to be
included in the final evacuation map. The sets of provenance dependency graphs
for CollabMap buildings, routes, and route sets are provided
by~\cite{Huynh2018} along with their corresponding \codeformat{trusted} or
\codeformat{uncertain} labels; they are named as \codeformat{CM-B},
\codeformat{CM-R}, and \codeformat{CM-RS}, respectively.

\subsubsection{\pg\ Simulator}\label{ssub:pokemongo}

\pg\ is a location-based augmented reality mobile game in which players, via a
mobile app, search, capture, and collect \pokemon\ that virtually spawn at
geo-locations around them~\cite{Paavilainen2017}. In this application, we
simulated part of the game's mechanics using NetLogo~\cite{Tisue2004}, a
multi-agent programmable modelling environment. It supports the concept of mobile
\emph{turtles} on a square grid of stationary \emph{patches}. Each turtle,
therefore, is located on a patch, essentially a 2-dimensional coordinate. The
turtles have individual states and a set of primitive operations available,
including rotating and moving. The simulator has turtles which represent the
geolocated \pokemons\ and \pokestops. These, however, do not move, and the
\pokemon\ are spawned only for a period of time. Other turtles represent the
players, which are assigned randomly to one of the three teams in the \pg\ game:
\textlabel{Valor}, \textlabel{Mystic}, and \textlabel{Instinct}; players move 
around to visit the \pokestops\ and to capture \pokemons. Simulation parameters 
include the initial number of \pokemons, the number of \pokestops, the number of 
players, and the maximum number of \pokemons\ a player can store.

During a simulation, if a player runs out of balls, it moves toward the closest
\pokestop\ and collects a random number of balls when
arriving there. Otherwise, the player chooses a \pokemon\ as a target, moves
toward it, and tries to capture it by ``throwing'' a ball at it.
However, if the player's \pokemon\ storage is full, it first has to dispose of
one of the \pokemons\ in storage before attempting a throw. A random number less
than 3,500 is generated in each throw: if it is larger than the \pokemon's
strength, the \pokemon\ is ``captured'', then put into the player's storage and
removed from the simulation. After each throw, successful or otherwise, the ball
is ``consumed'' and the player has one less.
During each simulation, we record game activities (collecting balls, capturing
and disposing of \pokemons) as provenance data.

We introduce into the simulation different strategies for each team on how
its players choose a \pokemon\ to \textbf{target} and to \textbf{dispose} of
when they need to. Players of the \textlabel{Valor} team always target for the
\emph{strongest} \pokemon, the \textlabel{Mystic} team the \emph{weakest}, and
the \textlabel{Instinct} team the \emph{closest}. These team-specific targeting
behaviours can be switched on/off in the simulation; when this is off, all
players target the \pokemon\ closest to them.
When space is needed in the \pokemon\ storage, players of the
\textlabel{Instinct} team dispose of the \emph{weakest} \pokemon, the
\textlabel{Mystic} team the \emph{earliest captured}, while the
\textlabel{Valor} team never disposes of a \pokemon.

We run two sets of 40 \pg\ simulations, with 30 players in each simulation. In
the first set, each team follows their individual \emph{targeting} strategy
while do not dispose of any \pokemon; in the second, all the teams target the
closest \pokemon\ while following their individual \emph{disposal} strategy
above. From the two experiments, we have two sets of 1,200 provenance
graphs; we call the first \codeformat{PG-T} and the second \codeformat{PG-D}.
Each of those graphs details the in-game actions taken by a particular
(simulated) player and is labelled with the player's team name (i.e.
\textlabel{Valor}, \textlabel{Mystic}, or \textlabel{Instinct}).

    %!TEX root = main.tex
\subsection{Methodology}\label{sub:methodology}

For each classification task, in order to ensure the robust evaluation of
provenance kernels' performance compared to that of existing graph kernels and
the PNA method, we carry out the following: balancing the input data set (if
unbalanced), training classifiers with provenance kernels (PK), generic graph
kernels, and provenance network metrics, measuring the performance of each
classifier, and comparing their performance. An overview of the full evaluation
pipeline, implemented in Python, is depicted in \fref{fig:pipeline}.

\begin{figure*}
  \centering
  \includegraphics[width=0.6\textwidth]{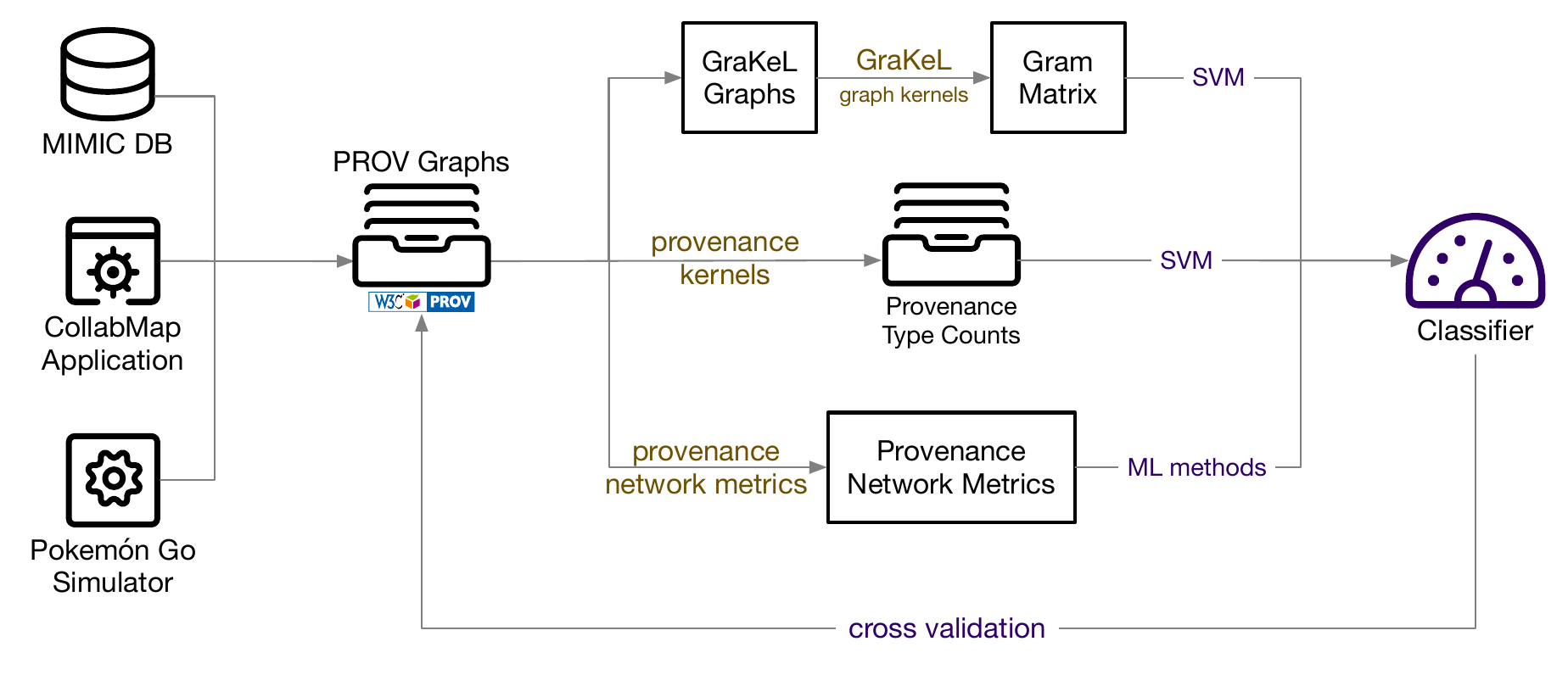}
  \caption{Overview of the evaluation pipeline. Six labelled provenance data sets from three applications are used to build classifiers using provenance kernels, generic graph kernels (from the GraKeL library), and provenance network metrics (PNA). Repeated 10-fold cross-validation is carried out to measure the classification accuracy of each method. We also measure the time each method takes to produce their kernels/network metrics (i.e.\ the yellow step).}\label{fig:pipeline}
\end{figure*}

\vspace{2pt}

\noindent \textbf{Data balancing} The MIMIC and CollabMap data sets are significantly
skewed, being originated from real-world human activities. Therefore, for those
data sets, we balance the number of samples in each class by selecting all the
samples in the minority class and randomly under-sampling the majority class to
produce a balanced data set. \tref{tab:datasets-overview} shows the number of
samples used in each classification task after balancing.

\noindent \textbf{Classification methods} In addition to building classifiers
for a classification task in question using provenance kernels,
we also build classifiers using existing graph kernels, implemented by
the Grakel library~\cite{Siglidis2020}, and the provenance network metrics
as proposed in the PNA method~\cite{Huynh2018}.
\begin{itemize}

\item Provenance Kernels: A provenance kernel is built on provenance types of
depth up to a specified level \(\depth \) which may include (1) \emph{only} the
PROV generic (application-agnostic) types, i.e. \ent, \act, and \ag, or (2)
application-specific types (such as the \qname{mimic:Patient} and
\qname{mimic:Ward} types shown in \fref{fig:mimic_example}) \emph{in addition
to} the PROV generic types. We test both provenance kernels using only generic
types and those including application types in our evaluation; we call the
former group \textlabel{PK-G} and the latter \textlabel{PK-A}. We also evaluate
provenance kernels for different levels of \(\depth \), from \(0\) to \(5\).
Hence, the methods we test in these two groups are: \textlabel{G0},
\textlabel{G1}, \ldots, \textlabel{G5}, \textlabel{A0}, \ldots, \textlabel{A5};
the first letter in their names denotes whether they use only PROV generic types
(\textlabel{G}) or not (\textlabel{A}) and the second denotes the specified
level \(\depth \); twelve PK-based methods are tested in total.

\item Graph Kernels: The graph kernels we test are Shortest Path (SP)~\cite{Borgwardt2005}, Vertex Histogram (VH)~\cite{sugiyama2015halting}, Edge
Histogram (EH)~\cite{sugiyama2015halting}, Graphlet Sampling (GS)~\cite{shervashidze2009efficient}, Weisfeiler-Lehman (WL)~\cite{shervashidze2011weisfeiler}, Weisfeiler-Lehman Optimal Assignment (WL-OA)~\cite{Kriege2016}, Hadamard Code (HC)~\cite{kataoka2016hadamard}, Ordered Dag Decomposition (ODD)~\cite{ODD_kernel}, Neighbourhood Hash (NH)~\cite{Neighbourhood-Hash-Kernel}, Neighbourhood Subgraph Pairwise Distance
(NSPD)~\cite{neighborhood-subgraph-pairwise}.\negmedspace\footnote{We also
tested other kernels provided by the Grakel library. However, they
either timed out or produced errors when processing graphs in our data sets and,
hence, are not included in our evaluation.} Similar to provenance
kernels, Weisfeiler-Lehman and Hadamard Code kernels can be computed up to a
specified level \(\depth \); we test those kernels with \(\depth \in [1, 5]\).
Hence, a total of 16 graph kernels are tested. As shown in \fref{fig:pipeline},
support-vector machines (SVM) are used to build classifiers from both provenance
kernels and generic graph kernels.
Among the tested graph kernels, the SP, GS, ODD, NH, NSPD, WL-OA kernels take a 
significantly longer time to run compared to the rest. We, therefore, for 
comparison purposes, put these methods in a group called \textlabel{GK-slow} and 
the remaining kernels in \textlabel{GK-fast}.

\item Provenance Network Analytics: The PNA method proposes calculating 22 
network metrics for each provenance graph and using those as the feature vector
for that graph. Such feature vectors can be readily taken as inputs by a variety
of ML algorithms. Since we are uncertain which algorithm works best with
provenance network metrics, we test the following classification algorithms over
the metrics: Decision Tree (DT), Random Forest (RF), K-Neighbour (KN), Gaussian
Naive Bayes (NB), Multi-layer Perceptron neural network (NN), and Support Vector
Machines (SVM). Hence, six methods are tested in total, all are implementations
by the Scikit-learn library~\cite{scikit-learn} using its default parameters
for them. This group of classifiers, which rely on provenance network metrics,
is called \textlabel{PNA}.

\end{itemize}

For any of the above methods that rely on SVM, its \(C\) parameter was optimally 
chosen from a grid search (with an inner cross-validation process) over the 
following  values \(\{10^{-4}, 0.001, 0.01, 0.1, 1.0, 10, 100, 1000 \}\) to give 
the best accuracy.

\begin{table*}
\caption{The classification tasks, the number of samples picked from each data set, and the number of application types present in each data set (in addition to PROV generic types).}\label{tab:datasets-overview}
\centering
\begin{tabular}{@{}r|c|ccc|cc@{}}
\toprule
Data set: & \codeformat{MIMIC} & \codeformat{CM-B} & \codeformat{CM-R} & \codeformat{CM-RS} & \codeformat{PG-T} & \codeformat{PG-D} \\ \midrule
Classification labels: & \codeformat{0}/\codeformat{1} & \multicolumn{3}{c}{\codeformat{trusted}/\codeformat{uncertain}} & \multicolumn{2}{c}{\codeformat{Valor}/\codeformat{Mystic}/\codeformat{Instinct}} \\
Random baseline: & 50\% & 50\% & 50\% & 50\% & 33\% & 33\% \\
Sample size: & 4,586 & 1,368 & 2,178 & 3,382 & 1,200 & 1,200 \\
No.\ application types: & 120 & 8 & 8 & 8 & 8 & 8 \\
\bottomrule
\end{tabular}
\end{table*}

\noindent \textbf{Performance metrics} The performance of each method is measured
by its accuracy in predicting the correct label of a provenance graph (i.e.\ the
number of correct prediction over the total number of samples), which is
provided with the above data sets.
To robustly measure the performance, 10-fold cross-validation is employed.
In particular, with all the available provenance graphs randomly split into 10
equal subsets, we perform 10 rounds of
learning; on each round, a \(\nicefrac{1}{10}\) subset is held out as the test
set and the remaining are used as training data.
This process is repeated 10 times; hence, 100 measures of accuracy are collected for each method per experiment.
In addition, to understand the computation cost of each method, we
measure the time it takes to produce provenance kernels, graph kernels, and
provenance network metrics (the yellow step in \fref{fig:pipeline}) given the
same data set used in a classification task.
The time measurements do not include the time taken in training the classifiers nor the time preparing the input GraKeL graphs.

\noindent \textbf{Comparing performance} Due to the large number of methods 
evaluated from the five groups (i.e.\ \textlabel{PK-G}, \textlabel{PK-A},
\textlabel{GK-slow}, \textlabel{GK-fast}, \textlabel{PNA}), we report here only
one best-performing method from each group, i.e.\ the one with the highest mean
classification accuracy within its group. We then compare the mean accuracy of
the best-performing PK-based methods (i.e.\ \textlabel{PK-G}, \textlabel{PK-A})
against those from the remaining three groups to establish whether PK-based
methods offer improved accuracy in the six classification tasks over existing
graph kernel and PNA methods. In order to ensure that our comparison results are
statistically significant, we carry out the Wilcoxon–Mann–Whitney ranks
test~\cite[Ch. 10]{Degroot2012}, also known as the Wilcoxon rank-sum test, when
comparing the accuracy measurements of two methods. If the test produces a
\emph{p}-value that is less than 0.05, we reject the null hypothesis that states
that the accuracy measurements are from the same distribution, i.e\ one method
performs statistically better than the other. Otherwise, both methods are
considered to have a similar level of performance. In addition, in real terms,
we disregard accuracy differences of less than 2\% and consider the
corresponding methods to have a similar level of performance.

    %!TEX root = main.tex
\subsection{Evaluation Results}\label{sub:results}

In this section, we report the performance of provenance kernels
(\textlabel{PK-G} and \textlabel{PK-A}) compared to that of existing generic
graph kernels (\textlabel{GK-slow} and \textlabel{GK-fast}) and the PNA method
(\textlabel{PNA}) across the classification tasks corresponding to the six
provenance data sets (\codeformat{MIMIC}, \codeformat{CM-B}, \codeformat{CM-R},
\codeformat{CM-RS}, \codeformat{PG-T}, and \codeformat{PG-D}). As previously
mentioned, for brevity, we only report the best-performing method in each group
in terms of their mean classification accuracy.

\begin{table}
\caption{Within each data set, we report the time cost of the best-performing method (shown in parentheses) from each comparison group relative to the time taken by the best-performing PK method (whose time cost shown as 1.0).} \label{tab:timings}
\centering
\begin{tabular}{@{}rlllll@{}}
\toprule
 & \textlabel{PK-G} & \textlabel{PK-A} & \textlabel{GK-fast} & \textlabel{GK-slow} & \textlabel{PNA} \\ \midrule
\codeformat{MIMIC} & 2.6 (G5) & 1.0 (A0) & 0.9 (WL2) & 201 (GS)   & 266 (SVM) \\
\codeformat{CM-B}  & 1.0 (G4) & 0.6 (A1) & 0.6 (HC3) &  10 (NSPD) &  90 (RF) \\
\codeformat{CM-R}  & 0.8 (G3) & 1.0 (A5) & 0.7 (WL5) &  53 (WLO4) & 107 (DT) \\
\codeformat{CM-RS} & 1.0 (G5) & 1.0 (A4) & 0.6 (WL3) &  88 (WLO5) & 111 (DT) \\
\codeformat{PG-T}  & 0.9 (G3) & 1.0 (A3) & 0.6 (WL5) &   5 (ODD)  & 113 (DT) \\
\codeformat{PG-D}  & 1.5 (G5) & 1.0 (A2) & 0.6 (WL5) &   5 (SP)   & 193 (SVM) \\ \bottomrule
\end{tabular}
\end{table}

\noindent\textbf{Computational costs}
Before delving into the performance of the five comparison groups, it is
pertinent to have an idea of the time costs incurred by them. Within each
classification task, using the time cost of the best-performing provenance
kernel as the relative time unit (i.e.\ 1.0), \tref{tab:timings} shows the relative time
costs of the best-performing method from each comparison group as multiples of
the chosen time unit. 

Formally, let \(\textlabel{X} \in \{\textlabel{PK-G}, \textlabel{PK-A}, \textlabel{GK-fast}, \textlabel{GK-slow}, \textlabel{PNA}\}\).
The entry for a given data set associated to the comparison group $\textlabel{X}$ is given by
$(\text{time-of-best-in-\textlabel{X}})$ divided by $(\text{time-of-best-in-\textlabel{PK}})$,
where `$\text{time-of-best-in-\textlabel{X}}$' stands for the time cost of the most accurate method in $\textlabel{X}$, whereas 
`$\text{time-of-best-in-\textlabel{PK}}$' stands for the time taken by the most accurate method in $\textlabel{PK} \in \{\textlabel{PK-G}, \textlabel{PK-A}\}$.

We also plot those time costs in \fref{fig:timings} using the logarithmic scale
to highlight their differences. Across the data sets, the \textlabel{PK-G} and
\textlabel{PK-A} methods take somewhat a similar time to produce the provenance
kernels from the same set of provenance graphs. Their differences are mainly due
to the different $\depth$ levels (of type propagation). We observe more
variation in relative time costs of the \textlabel{GK-fast} group's methods, but
they stay in the same magnitude of scale.
The best-performing graph kernels in the
\textlabel{GK-slow} group, however, take between 5x to 200x longer than the
baseline PK method to compute. The PNA methods are the slowest, taking 90x to
266x longer (to compute the provenance network metrics for the same set of
provenance graphs). Understanding the computational cost of each method, in
addition to its classification performance, will be useful when deciding whether
it is suitable for a given classification task.

\begin{figure}
  \centering
  \includegraphics[height=1.8in]{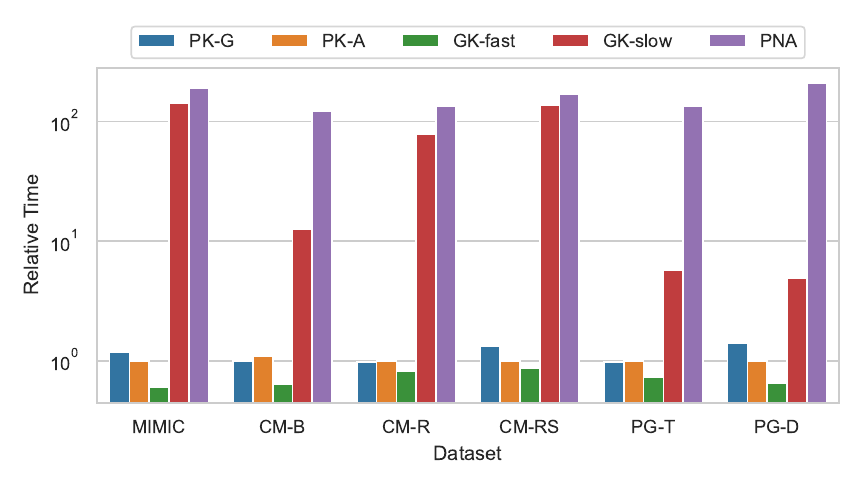}
  \caption{The relative time costs of the best-performing methods reported in \tref{tab:timings} plotted on the log scale.}\label{fig:timings}
\end{figure}

\begin{figure}
  \centering
  \includegraphics[height=1.8in]{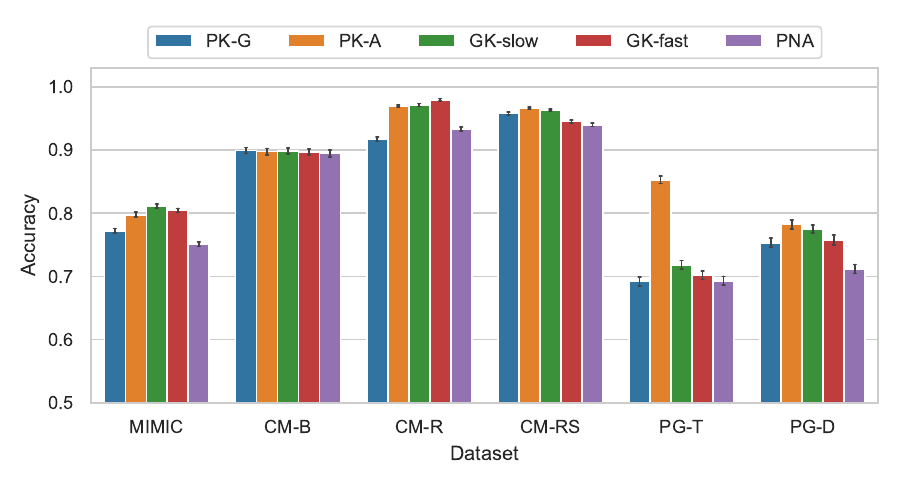}
  \caption{The mean classification accuracy of the best-performing provenance kernels, generic graph kernels, and PNA methods across the six classification tasks. The error bars show the 95-percent confidence intervals.}\label{fig:bestperformers}
\end{figure}

\noindent\textbf{Classification accuracy}
\fref{fig:bestperformers} plots the mean accuracy of the best-performing
provenance kernels compared against that of the best-performing graph kernels
and the best-performing PNA methods in the six classification
tasks.
(see \tref{tab:timings} for the identifier of the
best-performing method in each group, shown there in parentheses). 
At first
glance, in each task, the accuracy levels attained by the five comparison groups
are broadly close and significantly above the random baseline. This shows that
the proposed provenance types employed by PK-based methods, the graph
information relied on by GK methods, and the provenance network metrics used in
PNA methods can all serve effectively as predictors for these classification
tasks. However, their contributions to the accuracy of the corresponding
classifiers vary. To account for statistical variations, we carry out the
Wilcoxon–Mann–Whitney ranks tests to compare the accuracy level of the best
PK-based method with that of another comparison group in each classification
task. \tref{tab:summary} presents the results of those tests where an ``='' sign
indicates that the difference in accuracy is either less than 2\% or is not 
statistically significant; otherwise, a positive/negative value represents the 
accuracy gain/loss attained by the best PK-based method compared to the best 
method from the corresponding comparison group.

\noindent\textbf{Comparison with Graph Kernels}
If computational/time cost is not a consideration, we find that the
\textlabel{GK-slow} group generally outperforms the \textlabel{GK-fast} group
(see green bars vs.\ red bars in \fref{fig:bestperformers}). Compared to the
\textlabel{GK-slow} group, PK-based methods, however, yield similar levels of
accuracy across the tested classification tasks with the exception of
\codeformat{MIMIC} where the Graphlet Sampling kernel yields 2\% more accurate
classifications (at 200x more time cost) and \codeformat{PG-T} where the best
PK-based method is 14\% more accurate than the best \textlabel{GK-slow} method.
In terms of computation costs, it should be
noted that the best \textlabel{GK-slow} methods take 5x to 200x more
time
than their PK-based counterparts (see \tref{tab:timings}).
Compared to the \textlabel{GK-fast} group, \tref{tab:summary} shows that
PK-based methods are more accurate in two out of six classification tasks and
perform similar the remaining four tasks.
Hence, under time constraints (that disqualifies graph kernels in the
\textlabel{GK-slow} group), the proposed provenance kernels overall outperform 
the tested graph kernels in the \textlabel{GK-fast} group.

\begin{table}
\caption{Summary of the accuracy differences between the best-performing PK-based method and the best-performing method in the \textlabel{GK-slow}, \textlabel{GK-fast}, and \textlabel{PNA} groups. An ``='' sign means the accuracy difference is not significant; while a positive/negative value shows how much the PK-based method outperforms/under-performs the corresponding GK/PNA method, respectively.}%, when the difference is statistically significant.}
\label{tab:summary}
\centering
\begin{tabular}{@{}rcccccc@{}}
\toprule
Data set: & \codeformat{MIMIC} & \codeformat{CM-B} & \codeformat{CM-R} & \codeformat{CM-RS} & \codeformat{PG-T} & \codeformat{PG-D} \\ \midrule
\textlabel{GK-slow} & -2\% & =    & =    & =    & +14\% & = \\
\textlabel{GK-fast} & =    & =    & =    & =    & +15\% & +3\% \\
\textlabel{PNA}     & +3\% & =    & +3\% & +2\% & +16\% & +7\% \\ \bottomrule
\end{tabular}
\end{table}

\noindent\textbf{Comparison with PNA methods}
We also report in \tref{tab:summary} the accuracy differences between the best
PK-based methods compared to their PNA counterparts across the six
classification tasks. It shows that PK-based methods outperform in all tasks
with the exception of \codeformat{CM-B} where both groups perform at a similar
level of accuracy.
At the same time, given the significant penalty in computation cost incurred by
PNA methods in calculating network metrics (90--266 times, see
\tref{tab:timings}), PK-based methods prove to be better candidates for
analysing provenance graphs.
Moreover, the PNA method, in some of the tasks, employs obscure metrics
like the average clustering coefficients predominantly in the trained decision models,
making it a challenge to understand why certain classification is decided, even
with an interpretable model such as a decision tree classifier. In the following 
section, we show how provenance types used by the proposed provenance kernels
can afford us better interpretability %with respect to classification tasks
compared to the network metrics employed by PNA methods.

\vspace{-0.3cm}
\section{Explainability}\label{sec:explainability}
    %!TEX root = main.tex

In the evaluation of provenance kernels (Section~\ref{sec:evaluation}), Support Vector Machines are used to learn from the kernels' feature vectors, which are counts of provenance types, and perform classification tasks.
Techniques such as SVMs, however, are known as black-box models when considering their ability to provide explanations of their predictions.
In this section, we use LIME~\cite{lime}, short for Local Interpretable Model-Agnostic Explanations, to illustrate how it can help identify provenance types that are most influential in classification decisions and, from those, gain insights into classifiers built on provenance kernels. 

LIME aims to explain decisions of any classifier by introducing local perturbations of input data, through which it learns a linear model which is locally faithful to the instance to be explained.
For example, for tabular input data, it changes the value of a given feature, tests the perturbed feature vector with the same classifier, and checks whether such a change affects the prediction probabilities. 
The steeper the decrease in the prediction probability, the higher the contribution of that feature towards to a specific prediction.

Although provenance types alone do not explain the entirety of a process, they provide practitioners with a tool to better understand the decision process of graph classifiers.
We outline the steps to explain classification decisions using provenance kernels as follows. 
\begin{questions}
    \item\label{e1} \textbf{Important Types Identification:} Provenance kernel is an explicit graph kernel, i.e., we are able to inspect the feature vector used in the classification of each given graph. This, in turn, provides us with the set of provenance types that were present and used in such predictions. From the explanations for each instance produced by LIME, we are able to aggregate the importance of each feature across the entire data set and identify which provenance types were most influential in a given classification task. 

    \item\label{e4} \textbf{Instance Retrievability:} Once provenance types of interest are identified as being associated with a certain prediction, we are then able to retrieve original graph instances that contain patterns defined by one of the identified types. This is done by the subtasks (1) \textit{Graph instance identification}: searches which graph has a given feature in its feature vector; and (2) \textit{Subgraph retrieval}: extracts a subgraph that matches a provenance type in a given graph instance. Both combined may give us further insights into why a prediction was made as well as highlighting all provenance graphs that share the same pattern. 
    \item\label{e5} \textbf{Instance Description:} Once an instance of an interested pattern is found, we paraphrase such instance in a natural language by leveraging the descriptive capabilities of provenance vocabularies. 
\end{questions}
 
We illustrate the above steps with the \codeformat{CM-B} graphs.
First, we split it into a train set and a test set (at a ratio of 8:2).
We then train an SVM classifier on the train set.
For each graph in the test set, we record the importance of each feature from  LIME-generated explanations. 
The explanations were set up in a way such that the number of occurrences of each type in a particular graph was split into intervals, acting as a new (binary) feature.\footnote{The number of occurrences of a given type is denoted FA$x$\_$y$, where $x$ is the type's depth, and $x\_y$ is a type of depth $x$ (\ref{e1}).
Note that FA$x$\_$y$ are therefore non-negative integers.
Also, $x \leq 2$ since we are considering vectors in the context of types which depth is up to $2$.} For example, in Table \ref{tab:feature-importances}, we highlight the importance of features $\text{FA1\_1} \leq 2$, which indicates that the type $\text{FA1\_1}$ (see \tref{tab:decision-tree-type-definitions}) occurring $0$, $1$ or $2$ times a in single graph associates this graph as being considered \codeformat{trusted} by the classifier. The quantification of this association (e.g. $0.322$) is related to the drop in probability of being considered \codeformat{trusted} when such a feature is artificially removed from the feature vector. A negative importance indicates that a feature is associated with a \codeformat{uncertain} decision, as its removal increases the probability of a \codeformat{trusted} outcome. Note this analysis also allows us to conclude, for example, that a decision was made based on the \emph{absence}, rather than occurrence, of a given provenance type. 

\begin{table}
    \centering
    \caption{The average importances of provenance type intervals associated with parameters \text{FA2\_2}, \text{FA0\_5}, \text{FA1\_1} from \codeformat{CM-B} (see \tref{tab:decision-tree-type-definitions} for definitions) as produced by LIME.}
    \begin{tabular}{@{}rr@{}}
    \toprule
         Parameter &  Importance \\ 
         \midrule
          $\text{FA1\_1} \leq 2$ & $0.322$ \\
          $0 < \text{FA2\_2} \leq 2$ & $0.203$ \\
          $\text{FA0\_5} \leq 1$ & $0.0$ \\
          $\text{FA2\_2} = 0$ & $-0.204$ \\
          $\text{FA1\_1} > 2$ & $-0.320$ \\
        \bottomrule
    \end{tabular}
    \label{tab:feature-importances}
\end{table}

\begin{figure*}
\,\, 
\includegraphics[width = 0.43\textwidth]{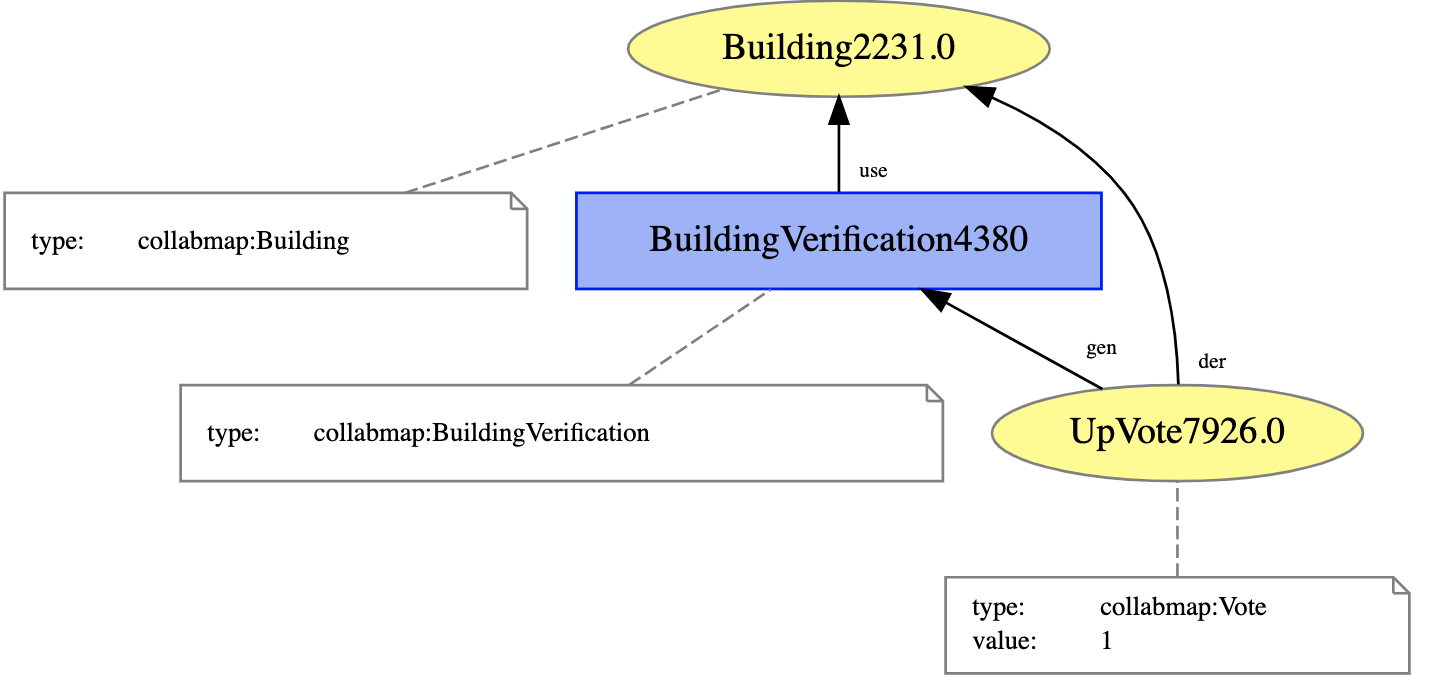} 
\hfill
\includegraphics[width = 0.43\textwidth]{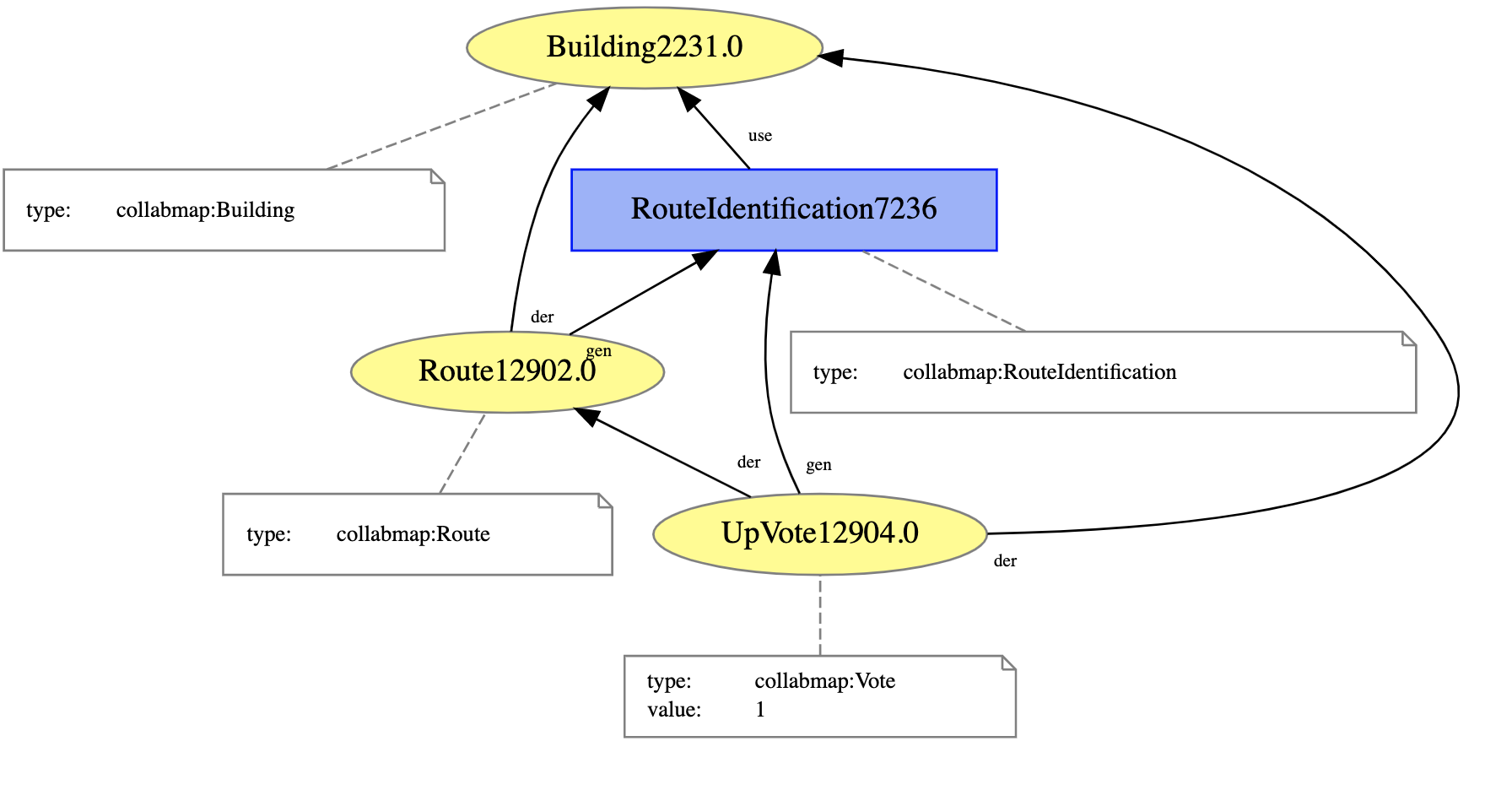} 
\,\,

\caption{An instance of a \text{FA1\_1} type (left), associated with node \textit{UpVote7926.0} and an instance of a \text{FA2\_2} type (right), associated with node \textit{UpVote12904.0}. Note that similar patterns associated to a DownVote would define the exact same provenance types, as both positive and negative votes are associated with the same provenance type \textit{collabmap:Vote}. } 
\label{fig:instance_retrievals}
\end{figure*}

\tref{tab:decision-tree-type-definitions} provides us with the type definitions associated with each one of the counters $\text{FA1\_1}$, $\text{FA2\_2}$, $\text{FA0\_5}$ (\ref{e1}).
Although $\text{FA0\_5}$ can simply be paraphrased as `a Buiding', when it comes to deeper types, however, their natural language descriptions need to make use of more general terms.
For example, $\text{FA2\_2}$ can be paraphrased as representing a `succession of two derivations from a Route or Building'.
Type descriptions alone, therefore, may not fully capture the complexity of graph patterns to give the user a broad understanding of their meaning.
For that reason, we propose the extraction of a subgraph associated with an occurrence of each one of the influential types in \tref{tab:feature-importances} (\ref{e4}).
Such a subgraph may by no means be the only graph pattern associated with a type but provides the user one concrete scenario of its occurrences.
\fref{fig:instance_retrievals}, for instance, depicts two subgraphs from \codeformat{CM-B} representing single instances of types \text{FA1\_1} and \text{FA2\_2}. Each is the induced subgraph from all nodes that are up to distance $\depth$, where $\depth$ is the depth (e.g. $1$ or $2$) of the provenance type of its root note (e.g. \textit{UpVote7926.0}, or \textit{UpVote12904.0}).
Finally, natural language paraphrasing of the retrieved subgraphs can be generated~\ref{e5}. For example, the following sentence was computationally generated from the graph on the right of \fref{fig:instance_retrievals} (i.e., an instance of \text{FA2\_2}):
\begin{quote}
    \textit {Route12902.0 was generated by RouteIdentification7236, UpVote12904.0 relates to Route12902.0., UpVote12904.0 was generated by RouteIdentification7236. RouteIdentification7236 used Building2231.0.}
\end{quote}

In \tref{tab:feature-importances}, the type $\text{FA1\_1}$ is associated with \codeformat{trusted} decisions when appears fewer than two times, and associated with \codeformat{untrusted} decisions if it is featured more frequently.
Note that in this case LIME aggregates the absence of type $\text{FA1\_1}$ and its single occurrence in the same interval.
In \fref{fig:instance_retrievals} this type refers to a vote on a building associated with a verification.
A natural way to interpret this is, since a \textit{Vote} can be either an up or down vote, we can take this to mean that a small number of votes likely means that they were up votes.
This is in line with how the CollabMap workflow~\cite{Ramchurn2013}: when there was a consensus with just a few votes (often up votes), the Building was declared \codeformat{trusted}.
Otherwise, if there was a dispute, more verifications, and hence votes, were requested, and thus it is more likely that the Building is deemed \codeformat{uncertain}.
In a narrative, we can say that a disputed decision is associated more with an \codeformat{uncertain} decision.
Moreover, a contrasting pattern is presented with the depth-$2$ type $\text{FA2\_2}$.
In particular, a non-zero number of occurrences was associated with \codeformat{trusted} decisions, whereas the absence of such type was deemed to be associated with a \codeformat{uncertain} one.
This once again is in line with the application's workflow: 
if a Building has not been declared \codeformat{trusted}, no routes are requested from crowd worker for the building. 

\begin{table}
    \centering
    \caption{The provenance types referred by the parameters \text{FA2\_2}, \text{FA1\_1}, \text{FA0\_5}. The prefix \textit{cm:} indicates an application-specific label from the CollabMap domain. \textit{cm:RouteID} represents a \textit{Route identification} type.}
    \begin{tabular}{@{}rl@{}}
    \toprule
         Param. & Provenance type represented \\ 
         \midrule
         \text{FA1\_1} & $\{\wdf, \wgb\}, \{\act,\ent, \textit{cm:Building}, \textit{cm:BuildingVerification}\}$ \\
        \text{FA2\_2} & $\{\wdf,\wgb\}, \{\wdf, \wgb, \used\}, \{\act, \ent, \textit{cm:Building}, \textit{cm:RouteID}\}$ \\
        \text{FA0\_5} & $\{\ent, \textit{cm:Building} \}$ \\
        \bottomrule
    \end{tabular}
    \label{tab:decision-tree-type-definitions}
\end{table}

Note that numbers of votes alone do not offer a strong explainable power.
This is because a vote can be either associated with a building verification or with a route identification, which have different implications when understanding the trustworthiness of a building.
For example, a high number of votes can be associated with a stronger likelihood of a building being deemed \codeformat{untrusted} if they are all associated with building verification activities.
On the other hand, the same high number of votes can be associated with a graph that has few building verification activities together with several route identification ones, which means the building was declared as \codeformat{trusted} by the workflow.
With that we can exemplify the importance of provenance types at higher depths as opposed to graph kernels that simply count the number of nodes of each type. 
Finally, the importance of the feature $\text{FA0\_5}$ was determined to be 0, which implies that the number of building offers no explanation power to the decision of the classifier.
Once again, this is in line with how the data set was constructed as \textit{Building} appears exactly once in each graph.

Through using an explainer, LIME, to build locally interpretable models for provenance kernel classifiers, the importance of each of the provenance types, including its constraints, can be determined, providing the user with a means to interpret the classifier's decisions.
We have shown that, in the context of the CollabMap building classification task, such interpretations reveal the logic captured from the data by the classifier (and later verified by us from checking the application's workflow).
Capitalised on the expressiveness of the provenance vocabulary and the patterns encoded by provenance types, insights into the logic of classifiers built on provenance kernels can be identified by following the steps we outlined above in this section.

\vspace{-0.3cm}
\section{Conclusions and Future Work}\label{sec:conclusion}
    %!TEX root = main.tex

With the growing adoption of provenance in a wide range of application domains, the efficient processing and classification of provenance graphs have become imperative.
To that end, we introduced a novel graph kernel method tailored for provenance data. 
Provenance kernels make use of provenance types, which are an abstraction of a node's neighbourhood taking into account edges and nodes at different distances from it.
A provenance type is associated with each node for each depth value $\depth$. 
A vector is then produced for each graph: it counts the number of occurrences of each (non-empty) provenance type associated to nodes in this graph up to a given depth $\depth$.
These feature vectors are then used with standard machine learning algorithms, such as SVMs and decision trees, as shown in the previous two sections.
The computational complexity of producing feature vectors for a family of graphs with a total of $M$ edges is bounded by $\OO (\depth^2 M)$.
Note that the provenance kernel method is applicable to graphs in any domain as long as both edges and nodes are categorically labelled.

In Section~\ref{sec:evaluation}, we compared provenance kernels against state-of-the-art graph kernels and the PNA method in supervised learning tasks with six data sets of provenance graphs.
We showed that provenance kernels are among the fastest methods and, among those, they show high, if not the highest, classification accuracies in the three applications we investigated. 
An important benefit brought about by provenance types is that they can be used with a white-box model as shown in Section~\ref{sec:explainability} to help us understand better how certain classification is made by the model.
We provided a sequence of steps to extract explanations from classifications tasks by inspecting the importance of different features, making use of the fact that provenance types capture narratives from sequential events in provenance graphs.
We thus show how provenance kernels and types may give us further insights into why a particular graph was classified in a particular way. Also, in the context of CollabMap data set, such explanations were in line of our knowledge with respect to how the graphs were constructed. 

A given provenance type, when considering only PROV generic node labels, may re-appear in provenance graphs recorded from different applications.
The extent of how many types overlap across application domains is an open question. 
In line with what has been proposed by \cite{ipaw_library}, an interesting line of future work 
is to create a library of types across different domains and investigate whether there is a correlation between the high occurrence of certain provenance types and the role they play in classification tasks. Moreover, it may be interesting to develop natural language descriptions of extracted instances of provenance types that capture the recursive aspect of types.

\subsubsection*{Code and Data}
The data used for this article, along with the associated experiment code, is publicly available at \url{https://github.com/trungdong/provenance-kernel-evaluation}.

\vspace{-1cm}
\begin{IEEEbiographynophoto}{David Kohan Marzagão}
is a Research Assistant at the Department of Engineering Science, University of Oxford. He is also affiliated to King's College London, from where he obtained his PhD in Computer Science
% He obtained his PhD in Computer Science from King's College London, to which he is currently affiliated.
\\
\\
\textbf{Dong Huynh} is a Research Fellow in the Department of Informatics, King’s College London.
He is the main developer of the PROV Python package  and ProvStore. 
 Previously, he developed computational models of trust and reputation for multi-agent systems.
\\
\\
\textbf{Ayah Helal} is a Lecturer in Computer Science at Exeter University and also affiliated to King's College London.
\\
\\
\textbf{Sean Baccas} is a Machine Learning Engineer at Polysurance. He obtained his MMath from the university of Durham.
\\
\\
\textbf{Luc Moreau} is a Professor of Computer Science and Head of the department of Informatics, at King's College London. 
Previously the co-chair of the W3C Provenance Working Group that produced the PROV standard.
\end{IEEEbiographynophoto}

\bibliographystyle{IEEEtran}
\bibliography{IEEEabrv,refs}
% \vfill
% \bibliography{refs}

\end{document}